\documentclass[journal]{IEEEtranTIE}
\usepackage{graphicx}
\usepackage{cite}
\usepackage{picinpar}
\usepackage{amsmath}
\usepackage{url}
\usepackage{flushend}
\usepackage[latin1]{inputenc}
\usepackage{colortbl}
\usepackage{subfigure}
\usepackage{soul}
\usepackage{multirow}
\usepackage{pifont}
\usepackage{color}
\usepackage{alltt}
\usepackage[hidelinks]{hyperref}
\usepackage{enumerate}
\usepackage{breakurl}
\usepackage{epstopdf}
\usepackage{pbox}
\usepackage{cancel}
\usepackage{caption2}
\usepackage{float}
\usepackage{makecell}
\usepackage{booktabs}
\usepackage{threeparttable}
\newtheorem{theorem}{\bf Theorem}
\newenvironment{proof}{{\it Proof}.}{\hfill $\square$\par}

\usepackage{amssymb}
\usepackage[boxed,ruled,commentsnumbered]{algorithm2e}  
\definecolor{mygray}{gray}{.9}

\begin{document}
\title{Propagating Asymptotic-Estimated Gradients for Low Bitwidth Quantized Neural Networks}
\author{Jun~Chen, Yong~Liu,~\IEEEmembership{Member,~IEEE}, Hao~Zhang, Shengnan~Hou, Jian~Yang

	\thanks{Jun~Chen, Yong~Liu, Hao~Zhang and Shengnan~Hou are with the Institute of Cyber-Systems and Control, Zhejiang University, Hangzhou,
		China, 310027, e-mail: yongliu@iipc.zju.edu.cn. Jian~Yang is with China Research and Development Academy of Machinery Equipment, Beijing, China, 100089.}
	\thanks{Yong~Liu is the corresponding author.}
}

\maketitle
	
\begin{abstract}
The quantized neural networks (QNNs) can be useful for neural network acceleration and compression, but during the training process they pose a challenge: how to propagate the gradient of loss function through the graph flow with a derivative of 0 almost everywhere. In response to this non-differentiable situation, we propose a novel Asymptotic-Quantized Estimator (AQE) to estimate the gradient. In particular, during back-propagation, the graph that relates inputs to output remains smoothness and differentiability. At the end of training, the weights and activations have been quantized to low-precision because of the asymptotic behaviour of AQE. Meanwhile, we propose a M-bit Inputs and N-bit Weights Network (MINW-Net) trained by AQE, a quantized neural network with 1-3 bits weights and activations. In the inference phase, we can use XNOR or SHIFT operations instead of convolution operations to accelerate the MINW-Net. Our experiments on CIFAR datasets demonstrate that our AQE is well defined, and the QNNs with AQE perform better than that with Straight-Through Estimator (STE). For example, in the case of the same ConvNet that has 1-bit weights and activations, our MINW-Net with AQE can achieve a prediction accuracy 1.5\% higher than the Binarized Neural Network (BNN) with STE. The MINW-Net, which is trained from scratch by AQE, can achieve comparable classification accuracy as 32-bit counterparts on CIFAR test sets. Extensive experimental results on ImageNet dataset show great superiority of the proposed AQE and our MINW-Net achieves comparable results with other state-of-the-art QNNs.
\end{abstract}


\markboth{IEEE JOURNAL OF SELECTED TOPICS IN SIGNAL PROCESSING}%
{}

\section{Introduction}

\IEEEPARstart{R}{ecently}, deep convolutional neural networks (DCNNs) have substantially dominated a variety of computer vision fields such as image classification~\cite{Krizhevsky2012ImageNet,Szegedy2014Going,Simonyan2015VeryDC}, face recognition \cite{taigman2014deepface,sun2014deep}, semantic segmentation \cite{long2015fully,chen2014semantic} and object detection \cite{girshick2015fast,ren2015faster}. However, the dazzling performance of DCNNs is at the expense of a large number of parameters and high computational complexity, which place high demands on the storage unit and greatly drag down the computational efficiency when considering to apply them to embedded devices and hardware platforms \cite{tulloch2017high,sharma2018bit,park2018energy}.

In order to deal with this problem, some studies~\cite{courbariaux2016binarized,courbariaux2015binaryconnect,rastegari2016xnor,zhou2016dorefa,li2016ternary,lin2017towards,zhu2016trained,park2018precision,choi2018bridging,choi2018pact,liu2018bi,zhuang2018towards} have specifically considered to reduce the parameters and improve the computational efficiency by using quantized neural networks (QNNs). In general, QNNs can be roughly divided into three categories: $1)$ simultaneously quantizing weights and activations whose model is the simplest such as \cite{courbariaux2016binarized}. $2)$ only quantizing weights such as \cite{courbariaux2015binaryconnect,li2016ternary}. $3)$ introducing scale factors in addition to the quantized values to fit floating-point numbers \cite{zhu2016trained,rastegari2016xnor,zhou2016dorefa}.

In principle, we can train neural networks by using gradient-based learning algorithm which is a well-known back-propagation algorithm \cite{rumelhart1988learning}. In this algorithm, the gradients propagate through the graph flow related between inputs and outputs. Although some researchers have extended the smoothness condition of the graph flow to non-smoothness condition \cite{glorot2011deep,goodfellow2013maxout}, the graph with a derivative of 0 almost everywhere is still not well processed when considering QNNs with low-precision weights and activations. For the graph flow like Dirac Delta Function, Bengio \emph{et al}. \cite{bengio2013estimating} give an unbiased estimator of gradient, where some components of the model allow a binary decision with stochastic perturbations \cite{spall1992multivariate}. This estimator only requires broadcasting the loss function, instead of using back-propagation algorithm. Another estimator is Straight-Through Estimator (STE) proposed by Hinton \emph{et al}. \cite{hinton2012neural}, which back-propagates by replacing the sign function with identity function. However, why can't we revert to the original smoothness condition to consider the problem of quantization?

In this paper, according to the smoothness condition of the graph, we propose the Asymptotic-Quantized Estimator that propagates the gradients through continuously differentiable graph, no need to worry about the derivative of 0. In our AQE, the neuron output function connected with the graph is related to the full-precision value and the quantized value. Under the guidance of the asymptotic behaviour of AQE, the output function gradually approaches the quantized value. When the training is completed, the output function is completely converted into the quantized value, and the graph flow at this time is ``non-smooth'' non-linearities. By using AQE, we introduce the MINW-Net with low-precision weights and activations. The main contributions of this article are summarized as follows:
\begin{enumerate}
	\item We propose a novel Asymptotic-Quantized Estimator based on the smoothness condition of the graph that can be used to train quantized neural networks. This estimator maintains the smoothness and differentiability of the graph flow during the training process and restores the non-smooth connection of the graph at the end of training. In particular, we demonstrate that our AQE will degenerate to STE and unbiased estimator when $\alpha=\frac{1}{2}$.
	\item We propose a highly efficient MINW-Net with low-precision weights and activations. When choosing 1-3 bits according to the actual situation, we can use XNOR or SHIFT operations instead of convolution operations to accelerate the inference, and the model has a small parameter capacity.
\end{enumerate}

The rest of this paper is organized as follows: Section~II summarizes related prior works on estimators and QNNs. Our AQE and MINW-Net are presented in Section~III. In Section~IV, we demonstrate the validities of our AQE and MINW-Net via comparable experiments. The conclusion is given in Section~V.

\section{Related work}

\subsection{Estimating or Propagating Gradients}
In general, the output $h_i$ of a stochastic neuron $i$ depends on the noise source $z_i$ and the input $a_i$, where the input $a_i$ is typically the affine transformation $a_{i}=b_{i}+\sum_{j} W_{i j} x_{i j}$ (the internal parameters are typically the bias and weights of the neuron) and $x_i$ is the input of neuron $i$ in neural networks,
\begin{equation}
h_{i}=f\left(a_{i}, z_{i}\right)
\label{output}
\end{equation}

Only the above continuous equation has a non-zero gradient with respect to $a_i$, then the process of calculating gradients using the back-propagation algorithm can be performed. However, considering that the input and output of the neuron are assigned to quantized values in quantized neural networks, the above continuous equation is transformed into the discrete function. At this point, this equation has a derivative of $0$ with respect to $a_i$ almost everywhere, the back-propagation algorithm cannot deal with this equation. Therefore, many scholars have carried out researches on this issue, and part of the important work is to construct effective estimators~\cite{fiete2006gradient,spall1992multivariate,bengio2013estimating,hinton2012neural}.

$\left. 1 \right)$ \emph{Unbiased Estimator:} Bengio \emph{et al}. \cite{bengio2013estimating} present the unbiased estimator for stochastic binary neuron based on the framework of \textbf{Eq}~(\ref{output}), which can be applied to binary neural networks. In this work, they consider the output $h_i$:
\begin{equation}
h_{i}=f\left(a_{i}, z_{i}\right)=\mathbf{1}_{z_{i}>\operatorname{sigm}\left(a_{i}\right)}
\label{unbiased}
\end{equation}
Where $z_{i} \sim U[0,1]$ is the uniform distribution and $\operatorname{sigm}(a_{i})=\frac{1}{1+\exp (-a_{i})}$ is the sigmoid function.

The gradient can be described as $g_{i}=\frac{\partial \mathbb{E}_{z_{i}, c_{-i}}\left[L | c_{i}\right]}{\partial a_{i}}$ by defining the $L=L\left(h_{i}, c_{i}, c_{-i}\right)$, where the loss function $L$ depends on $h_i$, $c_i$ the noise source that influences $a_i$, and $c_{-i}$ the noise source that does not influence $a_i$.

Considering $h_{i} \in\{0,1\}$ based on \textbf{Eq}~(\ref{unbiased}), so
\begin{equation}
\begin{split}
g_i=&\mathbb{E}_{c_{i}, c_{-i}}\left[\operatorname{sigm}\left(a_{i}\right)\left(1-\operatorname{sigm}\left(a_{i}\right)\right)\left(L\left(1, c_{i}, c_{-i}\right)-\right.\right.\\
&\left.\left.L\left(0, c_{i}, c_{-i}\right)\right)\right]
\end{split}
\end{equation}

Since the gradient $g_i$ cannot be computed by its derivative, they have constructed another gradient $\hat{g}_{i}=\left(h_{i}-\operatorname{sigm}\left(a_{i}\right)\right) \times L$ that satisfies $\mathbb{E}[\hat{g}_{i}]=g_i$. Thus, an unbiased estimation of the gradient can be used to update all parameters in the entire neural networks without the need for a derivative with respect to $a_i$.

$\left. 2 \right)$ \emph{Straight-Through Estimator:} Another useful method of training low-precision neural networks is the ``Straight-Through Estimator (STE)" introduced in Hinton's lectures~\cite{hinton2012neural}. This method simply back-propagates through the sign function ($1$ if the input is positive, $-1$ otherwise) as if it has been the identity function. Considering the output of activation function, that is sign function here, from a previous layer, $a_i$ in the framework of \textbf{Eq}~(\ref{output}),
\begin{equation}
h_i=\operatorname{sign}(a_i).
\label{sign}
\end{equation}
Since an estimator $g_{h_i}$ of the gradient $\frac{\partial L}{\partial h_i}$ has been obtained by back-propagation algorithm where $L$ is the loss function, the STE of $\frac{\partial L}{\partial a_i}$ with respect to the activation value $a_i$ is simply
\begin{equation}
g_{a_i}=g_{h_i} \mathbf{1}_{|a_i| \leq 1}.
\label{ste}
\end{equation}
The binary neural networks can be trained by back-propagation algorithm with STE.

\subsection{Low Precision Quantized Neural Networks}
In inference, considering the acceleration of neural networks with low-precision weights and activations, many studies have been devoted to achieve low-precision neural networks such as Binary Neural Network (BNN), Ternary Neural Network (TNN) and Quantized Neural Network (QNN)~\cite{courbariaux2015binaryconnect,courbariaux2016binarized,rastegari2016xnor,li2016ternary,zhou2016dorefa,zhou2017incremental,lin2017towards}. 

BinaryConnect~\cite{courbariaux2015binaryconnect} is the first article that summarizes the complete quantization process, whose weights in the forward and backward of DNNs are 1-bit fixed-point rather than 32-bit floating-point. The process they proposed allows hardware calculation to simplify multiplication operations into accumulation operations. Simultaneously, a large amount of storage space is reduced, and the performance of classification accuracy has not decreased on MNIST, CIFAR-10 and SVHN. BNN~\cite{courbariaux2016binarized} further converts the activation values to 1-bit based on the BinaryConnect, and simplifies many multiply-accumulate operations into bitwise operations. Since both the weights and the activations are quantized into $\{-1,1\}$, the gradient of the discrete neurons needs to be solved by STE. XNOR-Net~\cite{rastegari2016xnor} proposes a method of fitting floating-point numbers, which adds a scale factor based on binary value instead of simply taking the sign function. By combining binary value weights with scale factor, XNOR-Net can almost achieve the same performance as the full-precision model on AlexNet. The above works are either directly or indirectly related to the 1-bit weights, and Ternary Weight Network (TWN)~\cite{li2016ternary} proposes to quantize the weights into 2-bit, taking only three numbers $\{-1,0,1\}$ instead of four numbers, which can perform better than 1-bit weights on MNIST and CIFAR-10.

There is also another series of works on quantized neural networks, DoReFa-Net~\cite{zhou2016dorefa} sets the weights, the activations and the gradients to low-precision, whose advantage is that it can accelerate the calculation not only in the inference, but also in the training because the gradients are quantized. So DoReFa-Net is very suitable to train directly on the hardware platform. INQ~\cite{zhou2017incremental} proposes an incremental network quantization method, which quantizes the full-precision network to low-precision network by weight grouping, group quantization and retraining. By restoring a part of the floating-point weights to restore the performance of the model, the precision loss is mitigated. Since the existing methods with binary weights and activations can cause performance degradation, ABC-Net~\cite{lin2017towards} fits floating-point weights through a linear combination of multiple binary weights, which will reduce information loss. 

Above all, these quantized neural networks, once involved in quantizing the activations, are likely to use STE to estimate the gradient of the loss with respect to the activations. In other words, the performance of quantized neural networks based on STE is limited by the performance of the STE.

\section{MINW-Net}
In this section, we detail our MINW-Net, a QNN with low-precision weights and activations trained by AQE, because we note that STE may not be the best choice for quantized neural networks. First, we elaborate how to exploit Asymptotic-Quantized Estimator and then quantize the weights and the activations to low-precision through AQE.

\subsection{Quantized Estimator for Binary Neurons}
The Straight-Through Estimator above can deal with the Dirac Delta Function $\delta(\cdot)$ with a derivative of $0$ almost everywhere, but it can only propagate the gradient by the identity function. Now, let us consider this case, where the probability is a continuous function through stochastic neurons. The stochastic binary neurons of our model correspond to several binary decisions, and sign function is no longer used directly in the forward pass of neural networks. In the back-propagation algorithm, we assume that another continuous function is used to propagate the gradient instead of the identity function. In order to discuss this issue without loss of generality, in the framework of \textbf{Eq}~(\ref{output}), we propose our AQE where the output $\hat{h}_i$ of a neuron $i$ satisfies
\begin{equation} 
\hat{h}_i(\alpha, h_i(a_i), a_i) = \alpha h_i(a_i) + (1-\alpha) a_i
\label{ours}
\end{equation}
where $\alpha$ is an adjustable variable in the range of $(0,1)$ and $h_i(a_i)=\operatorname{sign}(a_i)$ as \textbf{Eq}~(\ref{sign}). According to \textbf{Eq}~(\ref{ours}), the output $\hat{h}_i(1, h_i(a_i), a_i)=\operatorname{sign}(a_i)$ degenerates into binary neuron if $\alpha=1$ and the output $\hat{h}_i(0, h_i(a_i), a_i)=a_i$ degenerates into full-precision neuron if $\alpha=0$.

The following discussion needs to be included in the framework of probability, so we must assume that there is an independent noise source $z_i$ that drives the stochastic samples first. In addition, considering that the value of $\operatorname{sign}(a_i - z_i)$ is $-1$ or $1$, there is no reason for the range of $z_i$ to exceed the output of function. So we follow the work~\cite{bengio2013estimating}, and we assume that $z_{i} \sim U[-1,1]$ is the uniform distribution.

With the above assumptions, we rewrite \textbf{Eq}~(\ref{ours}) as
\begin{equation} 
\hat{h}_i(\alpha, h_i, a_i, z_i) = \alpha \operatorname{sign}(a_i - z_i) + (1-\alpha) a_i
\label{our}
\end{equation}
During the back-propagation algorithm, we use following equation to calculate the gradient of weights ($ W^l_{ij}$ is the connecting between neuron $j$ at layer $l-1$ and neuron $i$ at layer $l$),
\begin{equation}
\frac{\partial L}{\partial W^l_{ij}}=\frac{\partial L}{\partial a^l_i} \hat{h}^{l-1}_j.
\end{equation}
Considering the expectation of the above equation, this equation can be re-expressed as 
\begin{equation}
\begin{split}
\frac{\partial}{\partial W^l_{ij}} \mathbb{E}_{z_i,c_{-i}}[L|c_i]=&\mathbb{E}_{z_i,c_{-i}}\left[\frac{\partial L}{\partial a^l_i} \hat{h}^{l-1}_j\bigg|c_i\right] \\
=&\mathbb{E}_{c_{-i}}\left[\mathbb{E}_{z_i}\left[\frac{\partial L}{\partial a^l_i}\right]\hat{h}^{l-1}_j\bigg|c_i\right]
\end{split}
\end{equation}
where $c_i$ is the noise source that influences $a_i$, $c_{-i}$ is the noise source that does not influence $a_i$, $\mathbb{E}_{c_{-i}}[\cdot]$ means the expectation over $c_{-i}$, and $\mathbb{E}[\cdot|c_i]$ means the expectation over everything besides $c_i$.

\begin{theorem}
\label{thm1}
Let us define $L=L(\hat{h}_i,c_i,c_{-i})$ where $\hat{h}_i$ follows \textbf{Eq}~(\ref{our}), then our AQE is equivalent to the STE when $\alpha=\frac{1}{2}$ (for brevity, we drop the indices of $l$).
\end{theorem}

\begin{proof}
\begin{equation}
\begin{split}
&\mathbb{E}_{z_i}\left[\frac{\partial}{\partial a_i} L\right]=\frac{\partial}{\partial a_i} \mathbb{E}_{z_i}[L] \\ 
&=\frac{\partial}{\partial a_i}[L(\hat{h}_i(\alpha,1,a_i)) P(a_i>z_i | a_i)+\\
&L(\hat{h}_i(\alpha,-1,a_i))(1-P(a_i>z_i | a_i))] \\ 
&=\frac{\partial P(a_i>z_i | a_i)}{\partial a_i}[L(\hat{h}_i(\alpha,1,a_i))-L(\hat{h}_i(\alpha,-1,a_i))]
\label{key}
\end{split}
\end{equation}

For $L(\hat{h}_i(\alpha,h_i,a_i))$, we can approximate it using the Taylor expansion. The output of a single neuron ($h_i=\pm1$) generally has only a small impact on the loss function~\cite{bengio2013estimating}, thus, $\frac{\partial L}{\partial h_i}\bigg|_{h_i=0} \gg \frac{\partial^{2} L}{\partial h_i^{2}}\bigg|_{h_i=0} \gg \frac{\partial^{3} L}{\partial h_i^{3}}\bigg|_{h_i=0}$. In other words, the $O(\cdot)$ is usually negligible when considering its impact on the last equation.
\begin{equation}
\begin{split}
&L(\hat{h}_i(\alpha,1,a_i))=L(\hat{h}_i(\alpha,0,a_i))+\frac{\partial L}{\partial h_i}\bigg|_{h_i=0}+\\
&\frac{\partial^{2} L}{\partial h_i^{2}}\bigg|_{h_i=0}+O\left(\frac{\partial^{3} L}{\partial h_i^{3}}\bigg|_{h_i=0}\right)\\
&L(\hat{h}_i(\alpha,-1,a_i))=L(\hat{h}_i(\alpha,0,a_i))-\frac{\partial L}{\partial h_i}\bigg|_{h_i=0}+\\
&\frac{\partial^{2} L}{\partial h_i^{2}}\bigg|_{h_i=0}+O\left(\frac{\partial^{3} L}{\partial h_i^{3}}\bigg|_{h_i=0}\right)
\label{L}
\end{split}
\end{equation}

For $\frac{\partial P(a_i>z_i | a_i)}{\partial a_i}$, we split it into two parts $|a_i| \leq 1$ and $|a_i| > 1$.
\begin{equation}
\begin{split}
&\frac{\partial P(a_i>z_i | a_i)}{\partial a_i}=\frac{\partial P(a_i>z_i | a_i)}{\partial a_i}+\frac{\partial P(a_i>z_i | a_i)}{\partial a_i}\\
&=\frac{\partial \int_{-1}^{1}\frac{1}{2}\,dz_i}{\partial a_i} \bigg|_{|a_i| > 1}+\frac{\partial \int_{-a_i}^{a_i}\frac{1}{2}\,dz_i}{\partial a_i} \bigg|_{|a_i| \leq 1}=\mathbf{1}_{|a_i| \leq 1}
\label{partical}
\end{split}
\end{equation}

Combining \textbf{Eq}~(\ref{L}) and \textbf{Eq}~(\ref{partical}), the \textbf{Eq}~(\ref{key}) can be derived as
\begin{equation}
\begin{split}
&\mathbb{E}_{z_i}\left[\frac{\partial L}{\partial a_i}\right]=\mathbf{1}_{|a_i| \leq 1}\left(2\frac{\partial L(\hat{h}_i)}{\partial h_i}\bigg|_{h_i=0}\right)\\
&=\left(2\frac{\partial L}{\partial h_i}\frac{\partial \hat{h}_i}{\partial h_i}\bigg|_{h_i=0}\right)\mathbf{1}_{|a_i| \leq 1}=2\alpha\frac{\partial L}{\partial h_i}\mathbf{1}_{|a_i| \leq 1}
\label{ste2}
\end{split}
\end{equation}
Let $\alpha=\frac{1}{2}$, then
\begin{equation}
\mathbb{E}_{z_i}\left[\frac{\partial L}{\partial a_i}\right]=\frac{\partial L}{\partial h_i}\mathbf{1}_{|a_i| \leq 1}
\label{STE}
\end{equation}
The \textbf{Eq}~(\ref{STE}) is equivalent to \textbf{Eq}~(\ref{ste}) that is the STE.
\end{proof}

\begin{theorem}
	\label{thm2}
	On the basis of \textbf{Theorem}~\ref{thm1}, our AQE is also equivalent to the unbiased estimator when $\alpha=\frac{1}{2}$.
\end{theorem}

\begin{proof}
\begin{equation}
\begin{split}
\frac{\partial L(\hat{h}_i)}{\partial a_i}&=\frac{\partial L}{\partial h_i}\frac{\partial h_i}{\partial \hat{h}_i(\alpha,h_i,a_i)}\frac{\partial \hat{h}_i(\alpha,h_i,a_i)}{\partial a_i} \\
&=\frac{\partial L}{\partial h_i} \frac{1}{\alpha} (1-\alpha)\underset{\alpha=\frac{1}{2}}{=}\frac{\partial L}{\partial h_i} 
\end{split}
\end{equation}
We have $\mathbb{E}\left[\frac{\partial L}{\partial a_i}\right]=\frac{\partial L}{\partial a_i}$ that is an unbiased estimator, because $\mathbb{E}\left[\frac{\partial L}{\partial a_i}\right]\underset{\alpha=\frac{1}{2}}{=}\frac{\partial L}{\partial h_i}$ based on \textbf{Eq}~(\ref{ste2}).
\end{proof}

\subsection{The Asymptotic Behaviour of the AQE}
We have demonstrated that our AQE is equivalent to STE and unbiased estimator when $\alpha=\frac{1}{2}$, whereas the output $\hat{h}_i(\alpha, h_i, a_i) = \alpha \operatorname{sign}(a_i) + (1-\alpha) a_i$ of stochastic neuron $i$ is still the continuous function, not the discrete function we need. We will use \textbf{Theorem}~\ref{thm3} to explain that the continuous function can approach to the discrete function, when the training process is completed.

\begin{theorem}
\label{thm3}
In the framework of the series, $\hat{h}_i(\alpha, h_i, a_i, z_i)$ will approach to $h_i = \operatorname{sign}(a_i-z_i)$ when the number of iterations is sufficient. An important assumption is that noise source $z_i$ during all iterations obeys the same uniform distribution.
\end{theorem}

\begin{proof}
We consider a back-propagation as an iterative process. Thus,
we define $a_i(n)$ as the $n$-th iteration of $a_i$, then $\hat{h}_i(\alpha, h_i, a_i(n))$ can be considered the next iteration of $a_i(n)$ which is equivalent to $a_i(n+1)$. We divide the input of stochastic neurons $a_i$ into two parts, which are $a_i^{> z_i}$ and $a_i^{\leq z_i}$. First consider the part of the neurons $a_i > z_i$ ($h_i=1$ at this time). Thus, the general terms of series $a_i^{> z_i}$ from $1$ to $n$ are written as follows based on \textbf{Eq}~(\ref{ours}),
\begin{equation} 
\begin{array}{lcc} 
& a_i^{> z_i}(2) - (1-\alpha) a_i^{> z_i}(1) = \alpha & (1) \\
& \vdots & \vdots \\
& a_i^{> z_i}(n) - (1-\alpha) a_i^{> z_i}(n-1) = \alpha & (n-1) \\
& a_i^{> z_i}(n+1) - (1-\alpha) a_i^{> z_i}(n) = \alpha & (n) \\
\end{array}
\label{series}
\end{equation}
Let $(n) + (1-\alpha) \times (n-1) + \cdots + (1-\alpha)^{n-1} \times (1)$, then we get the equation as follows, 
\begin{equation}
\begin{split} 
&a_i^{> z_i}(n+1) - (1-\alpha)^{n-1} a_i^{> z_i}(1) = \alpha [1 + (1-\alpha)+ \\
&(1-\alpha)^2 + \cdots + (1-\alpha)^{n-1}] \\
& = \alpha \frac{1-(1-\alpha)^n}{1-(1-\alpha)}=1-(1-\alpha)^n
\end{split}
\end{equation}

As the number of iterations increases, $a_i^{> z_i}(n+1)$ will asymptotically approach to $1$ (Similarly, $a_i^{\leq z_i}(n+1)$ will asymptotically approach to $-1$). Since the number of iterations is enough ($n \to \infty$) and $1-\alpha$ is in the range of $(0,1)$, the input of stochastic neurons $a_i(n+1)$ ($a_i^{> z_i}(n+1)$ and $a_i^{\leq z_i}(n+1)$) can be rewritten as $a_i(n+1) = h_i$ (in other words, $\hat{h}_i(\alpha, h_i, a_i(n))=h_i=\operatorname{sign}(a_i - z_i)$). With the guarantee of \textbf{Theorem}~\ref{thm3}, we can use our AQE to train the binary neural networks.
\end{proof}

\subsection{Low-precision Quantization of Weights and Activations}
From \textbf{Theorem}~\ref{thm3}, we know the fact that both weights and activations are continuous values with full-precision in the training process; both weights and activations are discrete values with low-precision in the inference process. In particular, the weights and activations involved in the calculation of loss function are full-precision, and those participated in the calculation of accuracy are low-precision. We don't care if the weights and activations in the training process are quantized, as long as the trained model is quantized in the inference. From the two parts of training and inference, we detail how to execute our MINW-Net with low-precision weights and activations using AQE. 

In QNNs, the activations are updated by our AQE in the training:
\begin{equation}
\begin{split}
&\textbf { Forward: } \hat{h}_i^l = \alpha h_i(a_i^l) + (1-\alpha) a_i^l \\ 
&\textbf { Backward: } \frac{\partial L}{\partial a_i^l}=2\alpha\frac{\partial L}{\partial h_i^l}\mathbf{1}_{|a_i| \leq 1},
\end{split}
\label{activation}
\end{equation}
the weights are updated by our AQE in the training:
\begin{equation}
\begin{split}
&\textbf { Forward: } \hat{W}_{ij}^l = \alpha h_i(W_{ij}^l) + (1-\alpha) W_{ij}^l \\ 
&\textbf { Backward: } \frac{\partial L}{\partial W_{ij}^l}=2\alpha\frac{\partial L}{\partial h_i^l}\mathbf{1}_{|W_{ij}^l| \leq 1}.
\end{split}
\label{weight}
\end{equation}

In inference, the weights and activations are quantized as
\begin{equation}
\begin{split}
&\hat{h}_i^l = h_i(a_i^l) \\ 
&\hat{W}_{ij}^l=h_i(W_{ij}^l).
\end{split}
\label{inference}
\end{equation}

For example in BNN, the output $h_i$ is sign function as \textbf{Eq}~(\ref{sign}), which returns a value from $\{-1,1\}$ that is consistent with \cite{courbariaux2015binaryconnect,courbariaux2016binarized,rastegari2016xnor}. In TNN, the output $h_i$ can be written as below:
\begin{equation}
h_i=h_i\left(a_i | \Delta\right)=\left\{
\begin{array}{ll}{0} & \text { if }\left|a_i\right| \leq \Delta  \\ 
{\operatorname{sign}(a_i)} & {\text { otherwise }}
\end{array}\right.
\label{tnn}
\end{equation}
where $\Delta$ is a positive threshold parameter, which returns a value from $\{-1,0,1\}$ based on \cite{li2016ternary}. 

The inference of the above two networks can be computed by bitwise operations without multiply-accumulate operations. When considering the bit width of weights and activations exceeds 2-bit, we no longer use the round-off integers scheme like \cite{zhou2016dorefa} because its inference will require multiply operations. In our QNN scheme, the function output $h_i$ is defined as below:
\begin{equation}
h_i=h_i\left(a_i | \Delta_j\right)=\left\{
\begin{array}{ll}{0} & \text { if }\left|a_i\right| \leq \Delta_q  \\ 
{\operatorname{sign}(a_i)} 2^{-p} & \text { if }\left|a_i\right| \leq \Delta_p
\end{array}\right.
\label{qnn}
\end{equation}
where $p$ is taken from $q-1$ to $0$ in turn, $\Delta_j$ represents $q$ positive threshold parameters, and $q=2^{n-1}-1$ for $n$-bit QNN. In this scheme, we can use shift operations instead of multiply operations to infer the entire network, because the weights and activations are powers of two.

\subsection{The Training and Inference Algorithm for MINW-Net}

We present the training and inference algorithm for our MINW-Net as \textbf{Algorithm}~\ref{al}, which is equally applicable to both convolutional layers and fully-connected layers, ignoring the details like Batch Normalization and Pooling layers (of course, they will be used in practice). According to $h_i$ in \textbf{Algorithm}~\ref{al} which can be chosen from \textbf{Eq}~(\ref{sign}), \textbf{Eq}~(\ref{tnn}) and \textbf{Eq}~(\ref{qnn}), we can execute BNN, TNN and QNN respectively. As a result, the forward function in this algorithm can obtain the acceleration by bitwise or shift operations.

\begin{algorithm} 
	\caption{Training our MINW-Net with M-bit Inputs (Activations) and N-bit Weights using AQE. The Activations and Weights are quantized based on \textbf{Eq}~(\ref{activation}) and \textbf{Eq}~(\ref{weight}) respectively.} 
	\SetKwInOut{KwIn}{\textbf{Require}}
	\SetKwInOut{KwOut}{\textbf{Ensure}}
	\KwIn{a minibatch of inputs and targerts $(a_0,a^*)$, learning rate $\eta$, and previous weights $W_l$} 
	\KwOut{the updated weights $W^{t+1}$} 

    \{1. Computing the gradients:\}
    
    \{1.1 Forward propagation:\}
	
	\For{$l=1$ to $n$} 
	{ 
		
		\uIf{`Training'}{
			$\hat{W}_{l}^{q} \leftarrow \alpha h_i(W_l)+(1-\alpha) W_l$\\
			$\tilde{a}_{l} \leftarrow$ $a_{l-1}^{q} \hat{W}_{l}^{q}$\\
			\uIf{$l<n$}{$a_{l}^q \leftarrow \alpha h_i\left(\tilde{a}_{l}\right)+(1-\alpha)\tilde{a}_{l}$}
			\textbf{end}}
		
		\uElseIf{`Inference'}{
		$\hat{W}_{l}^{q} \leftarrow h_i(W_l)$\\
		$\tilde{a}_{l} \leftarrow$ $a_{l-1}^{q} \hat{W}_{l}^{q}$\\
		\uIf{$l<n$}{$a_{l}^q \leftarrow h_i\left(\tilde{a}_{l}\right)$}	
		\textbf{end} 
		}	
		\textbf{end} 
	} 
	\{1.2 Backward propagation:\}
	
	Computing $g_{a_{n}}=\frac{\partial L}{\partial a_{n}}$ based on $a_n$ and $a^*$.
	
	\For{$l=n$ to $1$} 
	{ 
		\If{$l < n$}{
			$\tilde{a}_{l}  \leftarrow \max \left(-1, \min \left(1, \tilde{a}_{l} \right)\right)$
			
			$g_{\tilde{a}_{l}} \leftarrow g_{a_{l}^q}$
			}
		
		$g_{a^q_{l-1}} \leftarrow g_{\tilde{a}_{l}} \hat{W}_{l}^{q}$
		
		$g_{\hat{W}_{l}^{q}} \leftarrow g^T_{\tilde{a}_{l}} a_{l-1}^{q}$
	} 
	\{2. Updating the gradients:\}
	
		\For{$l=1$ to $n$} 
	{ 
		
		$W_{l}^{t+1} \leftarrow$update$\left(W_{l}, g_{\hat{W}_{l}^{q}}, \eta\right)$
		
		$W_{l}^{t+1}  \leftarrow \max \left(-1, \min \left(1, W_{l}^{t+1} \right)\right)$
	} 
\label{al}
\end{algorithm}

\subsection{Description of the Special Layers}
\begin{figure*}[hbtp]
	\centering
	\subfigure[]{\includegraphics[width=0.3\textwidth]{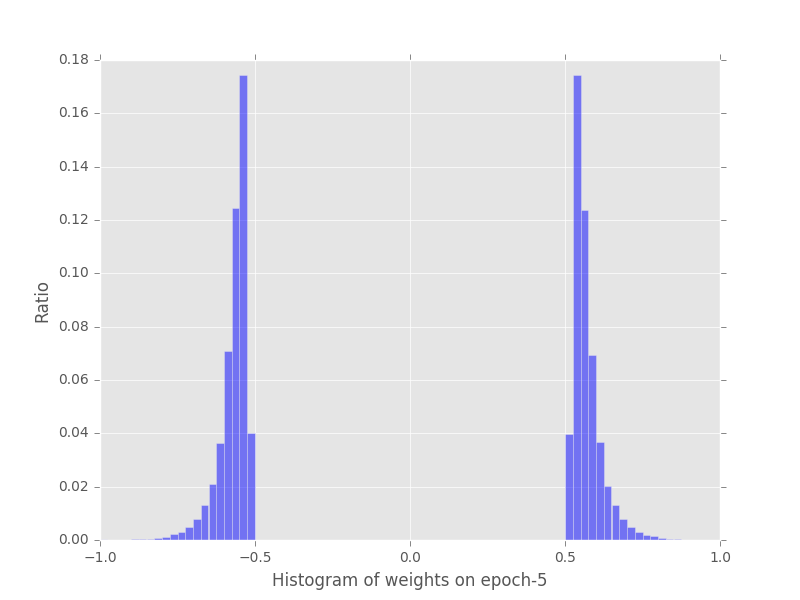}}
	\subfigure[]{\includegraphics[width=0.3\textwidth]{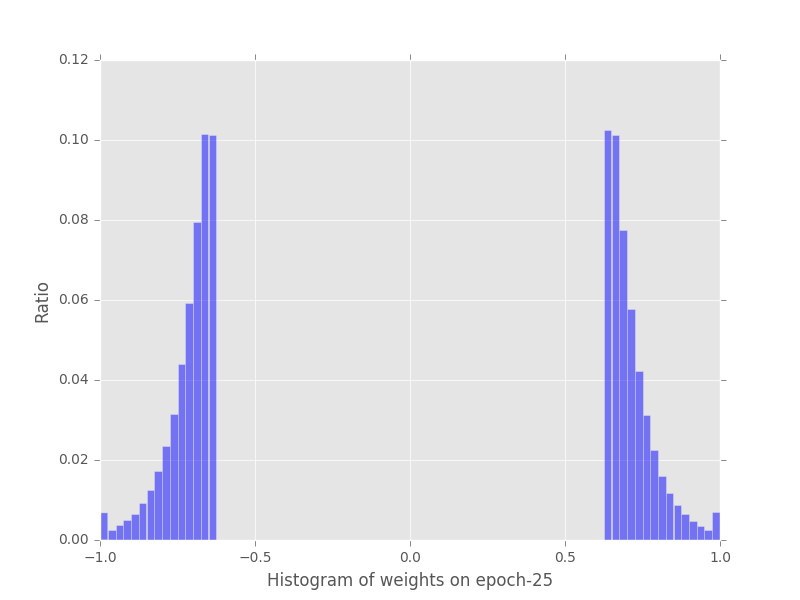}}
	\subfigure[]{\includegraphics[width=0.3\textwidth]{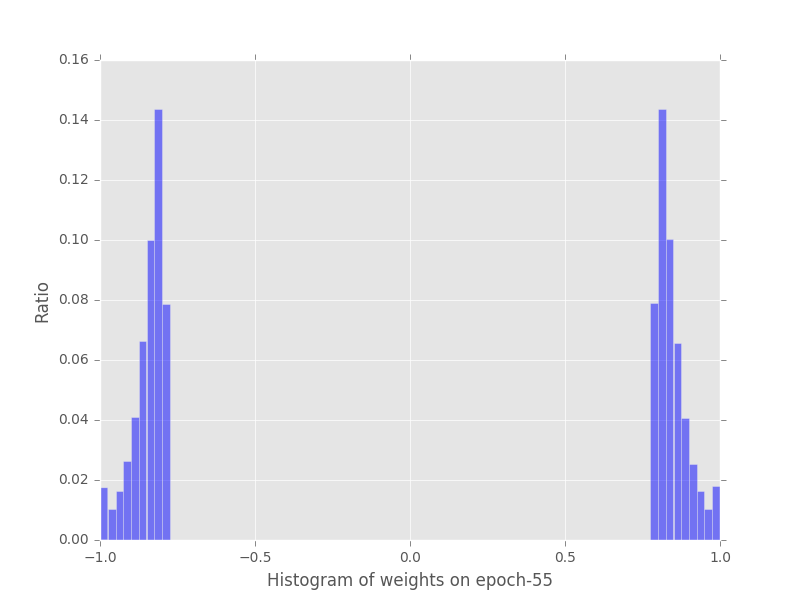}}
	\caption{Histogram of weights of layer ``conv2'' of BNN model at epoch 5, 25 and 55, where the y-axis is in logarithmic scale.}
	\label{weight-bnn} 
\end{figure*}

According to the related work of Courbariaux \emph{et al}. \cite{courbariaux2016binarized}, they do not quantize the input of the first convolutional layer, as they find that this behaviour will cause a significant degradation in classification accuracy compared to quantize other convolutional layers. On the other hand, the input of the first convolutional layer is the image itself, and quantizing the image recklessly will bring information loss. For these two reasons and observations, we don't deal with the input of the first convolutional layer. Of course, there is no reason not to quantize the input of other convolutional layers. Unlike the work of Zhou \emph{et al}. \cite{zhou2016dorefa}, they are not quantizing the output of the final fully-connected layer. In our opinion, no other layer is special (including all the input and output of convolutional and fully-connected layers) except the input of the first convolutional layer, so we have quantized them all.

In addition, in the related works of Zhou \emph{et al}. \cite{zhou2016dorefa} and Han \emph{et al}. \cite{han2015learning}, they are not quantizing the weights of the first convolutional layer because of the influence on classification accuracy. Nevertheless, from the consideration of computational complexity and parameter capacity, we quantize all the weights across the entire network.

Following the related works of Zhou \emph{et al}. \cite{zhou2016dorefa} and Courbariaux \emph{et al}. \cite{courbariaux2016binarized}, we also add the Batch Normalization (BN)~\cite{ioffe2015batch} in convolutional and fully-connected layers, as it is conducive to reducing the overall impact of the weight scale and accelerating the training. There are two convolutional layer constructs in MINW-Net, which are \{QuantizedConv-pooling-BN-HardTanh\} with Pooling layer and \{QuantizedConv-BN-HardTanh\} without Pooling layer respectively, where HardTanh is the ``hard tanh'' function: $\max \left(-1, \min \left(1, x\right)\right)$ whose role is reflected on $\mathbf{1}_{|a_i| \leq 1}$ in \textbf{Eq}~(\ref{STE}).

\section{Experiment}

\subsection{Histogram of the Asymptotic Behaviour}

\begin{figure*}[!htbp]
	\centering
	\subfigure[]{\includegraphics[width=0.3\textwidth]{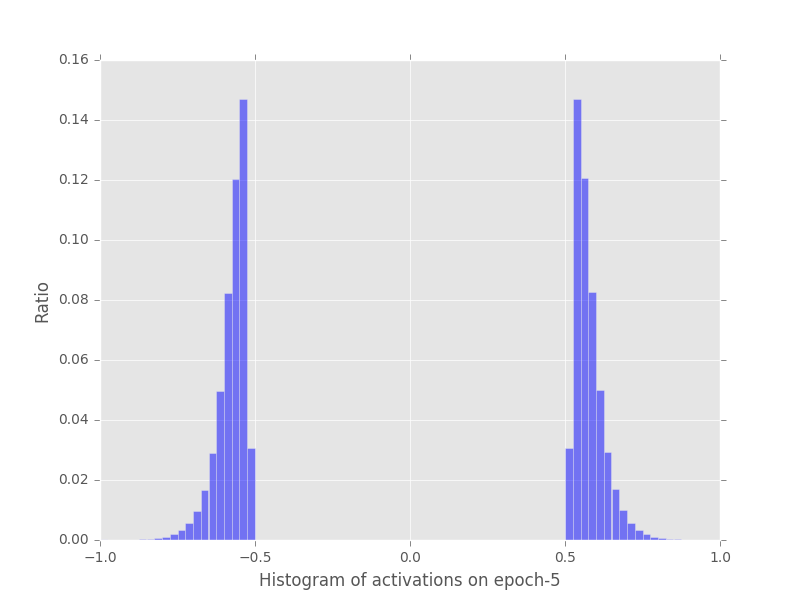}}
	\subfigure[]{\includegraphics[width=0.3\textwidth]{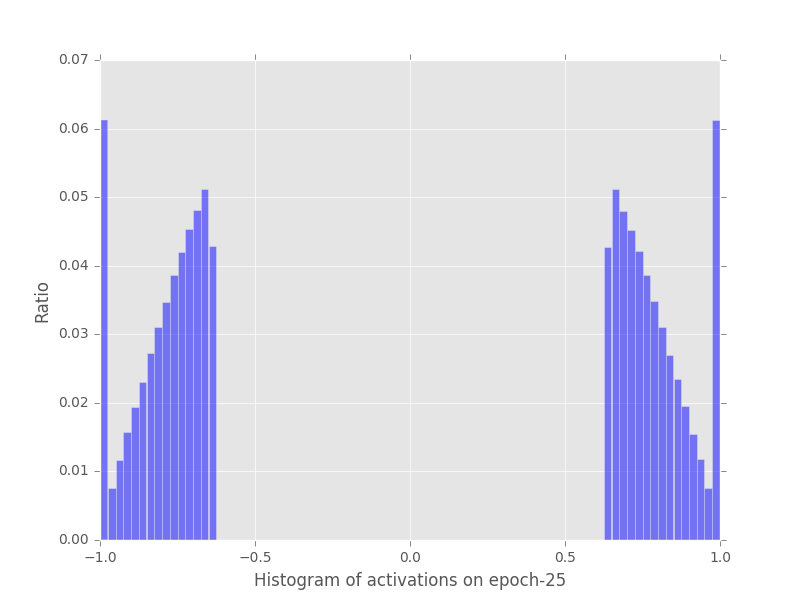}}
	\subfigure[]{\includegraphics[width=0.3\textwidth]{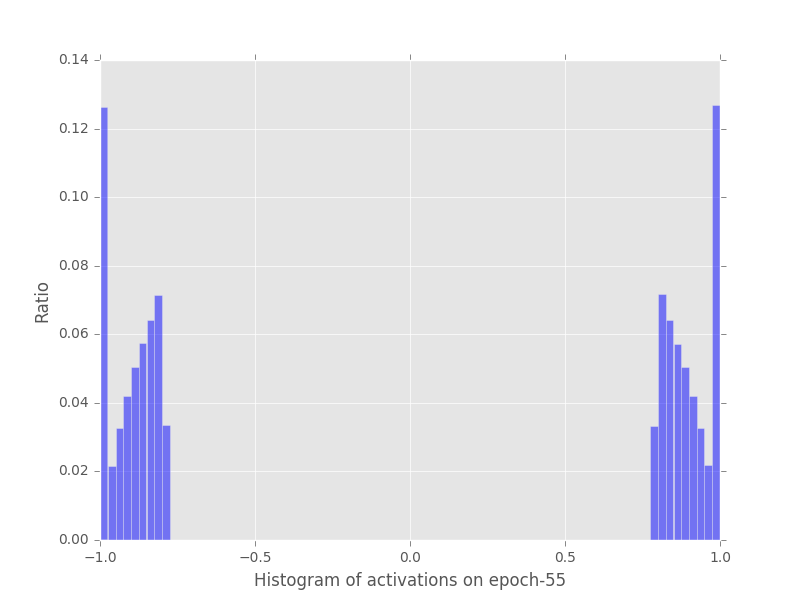}}
	\caption{Histogram of activations of layer ``conv2'' of BNN model at epoch 5, 25 and 55, where the y-axis is in logarithmic scale.}
	\label{activation-bnn} 
\end{figure*}
\begin{figure*}[!htbp]
	\centering
	\subfigure[]{\includegraphics[width=0.3\textwidth]{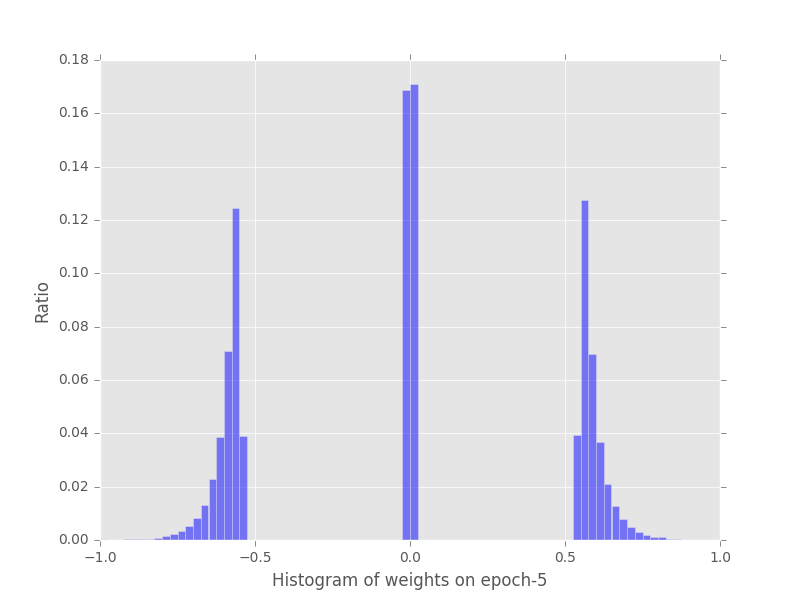}}
	\subfigure[]{\includegraphics[width=0.3\textwidth]{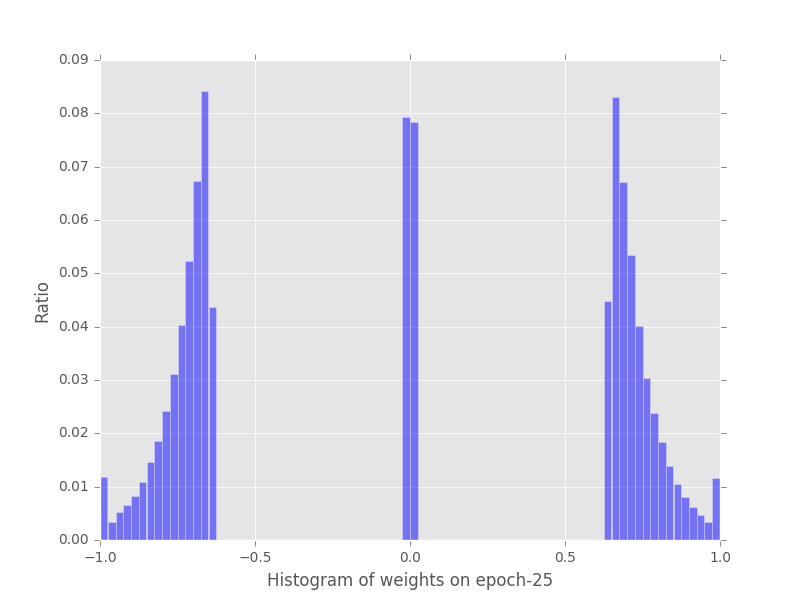}}
	\subfigure[]{\includegraphics[width=0.3\textwidth]{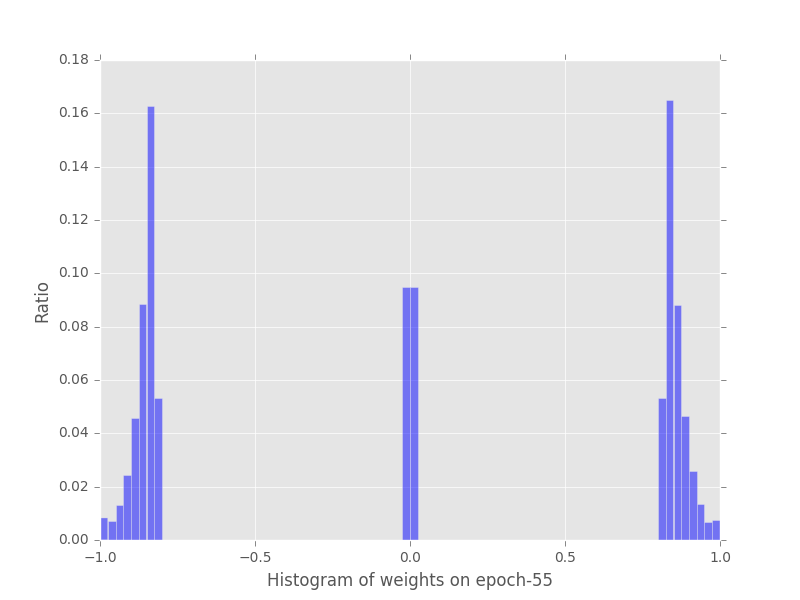}}
	\caption{Histogram of weights of layer ``conv2'' of TNN model at epoch 5, 25 and 55, where the y-axis is in logarithmic scale.}
	\label{weight-tnn} 
\end{figure*}
\begin{figure*}[!htbp]
	\centering
	\subfigure[]{\includegraphics[width=0.3\textwidth]{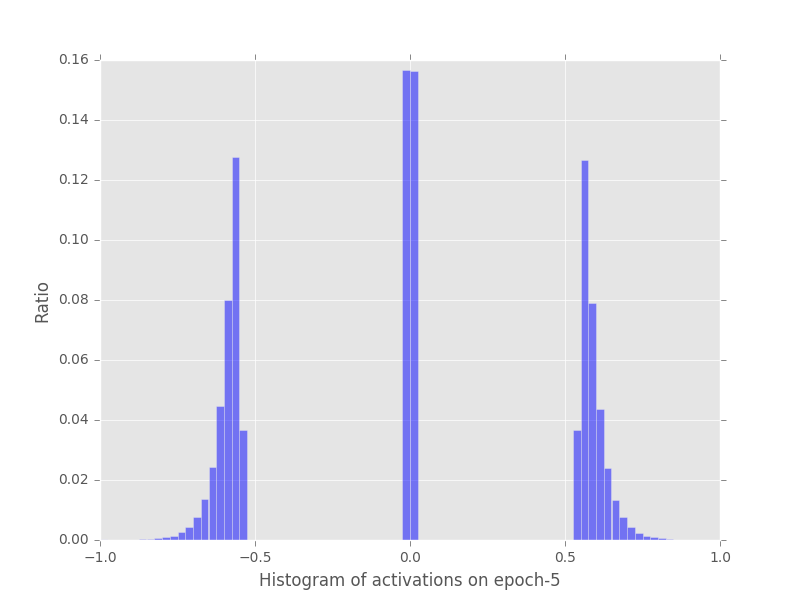}}
	\subfigure[]{\includegraphics[width=0.3\textwidth]{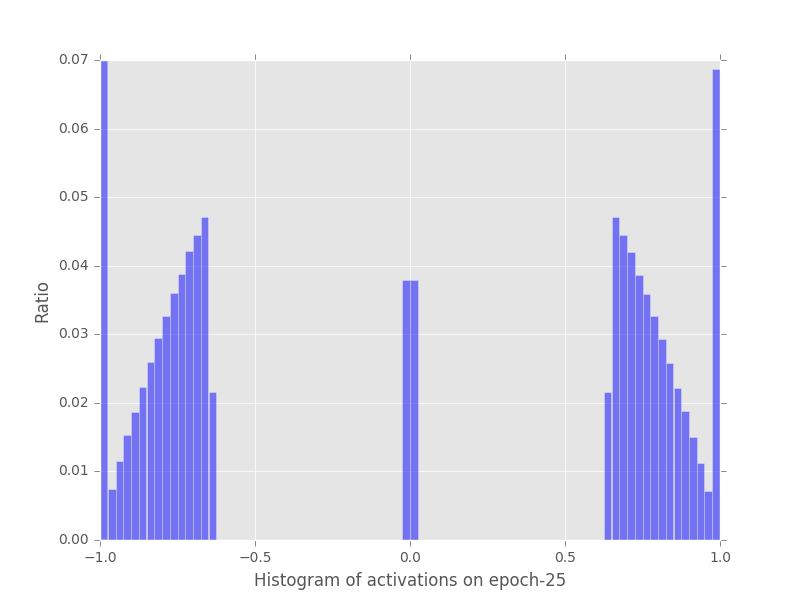}}
	\subfigure[]{\includegraphics[width=0.3\textwidth]{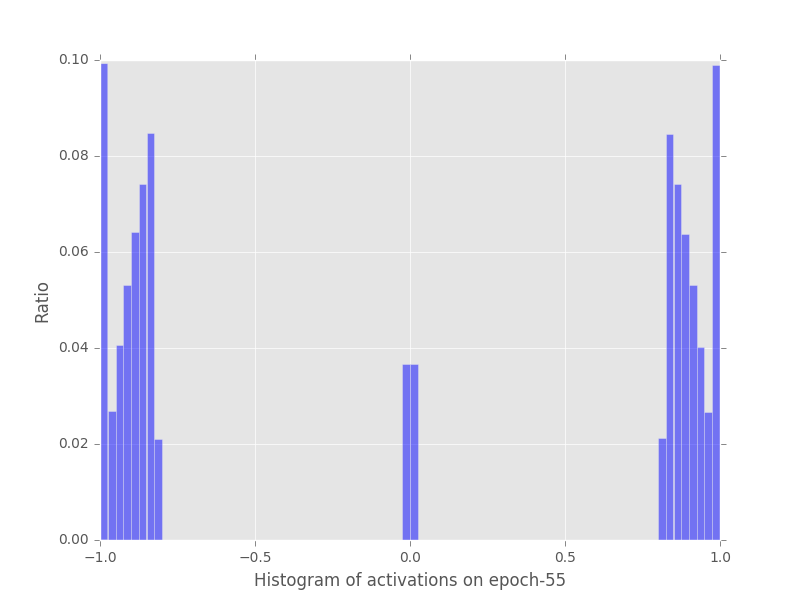}}
	\caption{Histogram of activations of layer ``conv2'' of TNN model at epoch 5, 25 and 55, where the y-axis is in logarithmic scale.}
	\label{activation-tnn} 
\end{figure*}

We have deduced our AQE by \textbf{Theorem}~\ref{thm1} and \textbf{Theorem}~\ref{thm2} in the previous section. However, the weights and activations used in these two theorems are still continuous values with full-precision. Then, we use \textbf{Theorem}~\ref{thm3} to prove that there is an asymptotic behaviour in the training process, that is, these weights and activations will asymptotically approach the quantized value. After the training, they will all be quantized.

Taking BNN and TNN for example, we observe the evolution of their distribution by extracting the weights and activations of different epochs in a convolutional layer during training. In BNN whose weights and activations will be constrained to $\{-1,1\}$, the \textbf{Fig}~\ref{weight-bnn}~\ref{activation-bnn} detail the distribution of weights and activations at epoch 5, 25 and 55. In TNN whose weights and activations will be constrained to $\{-1,0,1\}$, the \textbf{Fig}~\ref{weight-tnn}~\ref{activation-tnn} detail the distribution of weights and activations at epoch 5, 25 and 55. From the figures above, as the number of epochs increases, the weights and activations approach the quantized values gradually, and the proportion of the quantized values is increasing. This experiment also validates the asymptotic behaviour of \textbf{Theorem}~\ref{thm3}, which means that our asymptotic-quantized estimator is well defined.

\subsection{Comparison of performance between AQE and STE}
\begin{figure*}[htbp]
	\centering
	\subfigure[]{\includegraphics[width=0.48\textwidth]{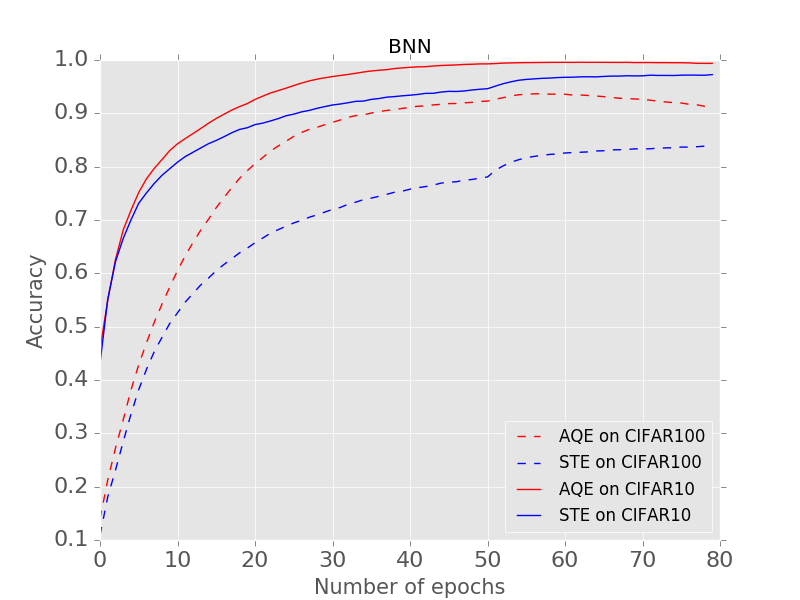}}
	\subfigure[]{\includegraphics[width=0.48\textwidth]{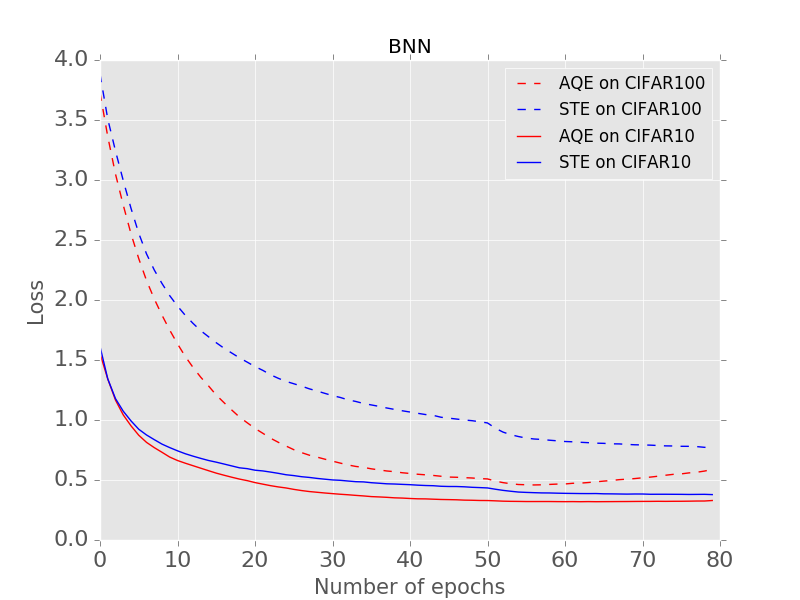}}
	\caption{Left: the accuracy curves for STE and AQE approaches applied to the CIFAR10 and CIFAR100 datesets on the training set. Right: the loss curves for STE and AQE approaches applied to the CIFAR10 and CIFAR100 datesets on the training set.}
	\label{train} 
\end{figure*}

\renewcommand\arraystretch{1.8}
\begin{table*}[htbp]
	\caption{Comparison of classification accuracy on the test set for CIFAR10 with different bitwidth in MINW-Net. We just remove the quantization layers when the precision is 32.}
	\label{cifar10} 
	\begin{center}
		\begin{tabular}{p{50pt}p{50pt}p{70pt}p{70pt}p{35pt}p{35pt}p{35pt}p{35pt}}
			\toprule
			\textbf{M-bit Inputs }& \textbf{N-bit Weights} & \textbf{Inference operation} & \textbf{Inference computational precision} & \textbf{Model A Accuracy} & \textbf{Model B Accuracy} & \textbf{Model C Accuracy} & \textbf{Model D Accuracy} \\
			\hline
			\hline
			1 & 1 & XNOR & 1-bit fixed & 0.884 & 0.851 & 0.847 & 0.844 \\
			\hline
			1 & 2 & XNOR & 1-bit fixed & 0.889 & 0.866 & 0.859 & 0.855 \\
			\hline
			1 & 3 & XNOR & 2-bit fixed & 0.897 & 0.873 & 0.868 & 0.867 \\
			\hline
			1 & 16 & XNOR ADDER & 16-bit floating & 0.901 & 0.876 & 0.870 & 0.871 \\
			\hline
			\hline
			2 & 1 & XNOR & 1-bit fixed & 0.891 & 0.875 & 0.868 & 0.867 \\
			\hline
			2 & 2 & XNOR & 1-bit fixed & 0.898 & 0.879 & 0.873 & 0.871 \\
			\hline
			2 & 3 & XNOR & 2-bit fixed & 0.903 & 0.888 & 0.888 & 0.881 \\
			\hline
			2 & 16 & XNOR ADDER & 16-bit floating & 0.909 & 0.890 & 0.891 & 0.887 \\
			\hline
			\hline
			3 & 1 & XNOR & 2-bit fixed & 0.890 & 0.889 & 0.880 & 0.880\\
			\hline
			3 & 2 & XNOR & 2-bit fixed & 0.895 & 0.892 & 0.887 & 0.888\\
			\hline
			3 & 3 & SHIFT & 2-bit fixed & $\mathbf{0.906}$ & $\mathbf{0.898}$ & $\mathbf{0.891}$ & $\mathbf{0.893}$ \\
			\hline
			3 & 16 & SHIFT ADDER & 16-bit floating & 0.911 & 0.901 & 0.896& 0.895\\
			\hline
			\hline
			16 & 1 & XNOR ADDER & 16-bit floating & 0.908 & 0.890 & 0.886 & 0.887 \\
			\hline
			16 & 2 & XNOR ADDER & 16-bit floating & 0.908 & 0.894 & 0.892 & 0.891 \\
			\hline
			16 & 3 & SHIFT ADDER & 16-bit floating & 0.911 & 0.901 & 0.904 & 0.903 \\
			\hline
			\hline
			32 & 32 & MAC & 32-bit floating & 0.918 & 0.910 & 0.912 & 0.910\\
			\bottomrule
		\end{tabular}
	\end{center}
\end{table*}

\renewcommand\arraystretch{1.8}
\begin{table*}[htbp]
	\caption{Comparison of classification accuracy on the test set for CIFAR100 with different bitwidth in MINW-Net. We just remove the quantization layers when the precision is 32.}
	\label{cifar100} 
	\begin{center}
		\begin{tabular}{p{50pt}p{50pt}p{70pt}p{70pt}p{35pt}p{35pt}p{35pt}p{35pt}}
			\toprule
			\textbf{M-bit Inputs }& \textbf{N-bit Weights} & \textbf{Inference operation} & \textbf{Inference computational precision} & \textbf{Model A Accuracy} & \textbf{Model B Accuracy} & \textbf{Model C Accuracy} & \textbf{Model D Accuracy} \\
			\hline
			\hline
			1 & 1 & XNOR & 1-bit fixed & 0.596 & 0.554 & 0.546 & 0.530 \\
			\rowcolor{mygray}1 & 1 & XNOR & 1-bit fixed & 0.596 & 0.552 & 0.535 & 0.515 \\
			\hline
			1 & 2 & XNOR & 1-bit fixed & 0.602 & 0.559 & 0.551 & 0.542\\
			\hline
			1 & 3 & XNOR & 2-bit fixed & 0.611 & 0.564 & 0.557 & 0.550\\
			\hline
			1 & 16 & XNOR ADDER & 16-bit floating & 0.620 & 0.569 & 0.560 & 0.555\\
			\hline
			\hline
			2 & 1 & XNOR & 1-bit fixed & 0.618 & 0.586 & 0.580 & 0.564\\
			\hline
			2 & 2 & XNOR & 1-bit fixed & 0.622 & 0.590 & 0.588 & 0.569 \\
			\hline
			2 & 3 & XNOR & 2-bit fixed & $\mathbf{0.631}$ & $\mathbf{0.603}$ & $\mathbf{0.592}$ & $\mathbf{0.582}$ \\
			\hline
			2 & 16 & XNOR ADDER & 16-bit floating & 0.633 & 0.607 & 0.595 & 0.585\\
			\hline
			\hline
			3 & 1 & XNOR & 2-bit fixed & 0.593 & 0.565 & 0.548 & 0.544 \\
			\hline
			3 & 2 & XNOR & 2-bit fixed & 0.606 & 0.582 & 0.565 & 0.560 \\
			\hline
			3 & 3 & SHIFT & 2-bit fixed & 0.618 & 0.601 & 0.581 & 0.573 \\
			\hline
			3 & 16 & SHIFT ADDER & 16-bit floating & 0.621 & 0.610 & 0.585 & 0.579\\
			\hline
			\hline
			16 & 1 & XNOR ADDER & 16-bit floating & 0.631 & 0.618 & 0.597 & 0.584\\
			\hline
			16 & 2 & XNOR ADDER & 16-bit floating & 0.635 & 0.622 & 0.608 & 0.598 \\
			\hline
			16 & 3 & SHIFT ADDER & 16-bit floating & 0.640 & 0.630 & 0.619 & 0.614 \\
			\hline
			\hline
			32 & 32 & MAC & 32-bit floating & 0.646 & 0.634 & 0.630 & 0.623 \\
			\bottomrule
		\end{tabular}
	\end{center}
\end{table*}

We use CIFAR10 and CIFAR100 datesets to explore the performance of AQE, where the network structure is the same BNN.

The two CIFAR datasets \cite{krizhevsky2009learning} consist of color natural images with 32$\times$32 pixels, respectively 50,000 training images and 10,000 test images, and we hold out 5,000 training images as a validation set from training set. CIFAR10 consists of images organized into 10 classes, and CIFAR100 into 100 classes. We adopt a standard data augmentation scheme (random corner cropping and random flipping) that is widely used for these two datasets \cite{lin2013network,lee2015deeply,springenberg2014striving,srivastava2015training}. We normalize the images using the channel means and standard deviations in preprocessing.

In this experiment, we perform the classification accuracy by comparing the AQE with STE on CIFAR10 and CIFAR100. The evaluation of the experiment is based on ConvNet, which is used by Courbariaux \emph{et al}. \cite{courbariaux2016binarized}, where the stochastic neuron output is \textbf{Eq}~(\ref{sign}) in STE and \textbf{Eq}~(\ref{ours}) in AQE respectively. On the \textbf{Fig}~\ref{train}, we test the two estimators using the same conditions\footnote{We need to correct an ambiguity. In the previous section, we have mentioned that our AQE uses the weights and activations with full-precision on inference during the training process, as they will participate in the loss calculation of the updated parameters. But in this experiment, in order to be consistent with STE, we will use the binarized weights and activations to obtain accuracy and loss.} (including learning rate, network structure and number of epochs), and conclude that our AQE performs better in both accuracy and loss than the STE.

\subsection{Low-bitwidth Exploration}

In this section, we use AQE to explore the best configuration for different combinations of weights and activations precisions on CIFAR datasets.

\textbf{ConvNet:} The basic structure of ConvNet~ \cite{courbariaux2016binarized} consists of two layers, one of which is feature extraction layer. The input of each neuron is connected to the local receptive field of the previous layer, and the local features are extracted. The second is the feature mapping layer, Sigmoid or ReLU function is used as the activation function of ConvNet in feature mapping structure, which makes feature mapping have displacement invariance. Each convolutional layer in ConvNet is followed by a pooling layer used to calculate local average and quadratic extraction.

\textbf{ResNet:} The main idea of ResNet~\cite{he2016deep} is to add a direct connection channel to the network, namely Highway Network~\cite{srivastava2015highway}. The previous network structure is a non-linear transformation of the input, while Highway Network retains a certain proportion of the output of the previous layer. Similarly, ResNet allows the original input information to be passed directly to the later layers, whose idea paves the way for a very deep network (more than 100 layers) to train.

\textbf{DenseNet:} It is inspired by ResNet and Highway Network, and proposes to transform the input of each convolutional layer into the splicing of the output of all previous layers~\cite{huang2017densely}. Such a dense connection makes it possible for each layer to take advantage of all the features previously learned, without the need for repetitive learning. At the same time, it imitates the structure of ResNet to make the gradient spread better, which makes it more convenient to train deep networks.

\renewcommand\arraystretch{1.8}
\begin{table}[htpb]
	\caption{Comparison of test error on CIFAR10 (100) between 3-bit weights and 32-bit float weights, where the results are given based on ResNet and DenseNet.}
	\label{quantization_error} 
	\begin{center}
		\begin{tabular}{|c|c|c|c|c|}
			\hline
			\textbf{Network} & \textbf{Depth} & \textbf{Dataset} & \textbf{Bitwidth} & \textbf{Test error (\%)} \\
			\hline
			\multirow{4}{*}{\textbf{ResNet}~\cite{he2016deep}} & \multirow{4}{*}{110} & CIFAR10 & 32 (float) & 6.61\cr\cline{3-5}
			&& CIFAR10 & 3 & 7.05 (+0.44)\cr\cline{3-5}
			&& CIFAR100 & 32 (float) & 35.87\cr\cline{3-5}
			&& CIFAR100 & 3 & 37.16 (+1.29)\\
			\hline
			\multirow{4}{*}{\textbf{DenseNet}~\cite{huang2017densely}} & \multirow{4}{*}{100} & CIFAR10 & 32 (float) & 4.51\cr\cline{3-5}
			&& CIFAR10 & 3 &5.21 (+0.70)\cr\cline{3-5}
			&& CIFAR100 & 32 (float) & 22.27\cr\cline{3-5}
			&& CIFAR100 & 3 & 23.70 (+1.43)\\
			\hline
		\end{tabular}
	\end{center}
\end{table}

For the convolutional and fully-connected layers in our MINW-Net, we have listed M-bit Inputs and N-bit Weights, inference operation, inference computational precision and classification accuracy of different models in \textbf{Table}~\ref{cifar10}~\ref{cifar100}. When the bitwidth of inputs and weights are both $1$, our MINW-Net is degenerated to BNN, and its inference operation is XNOR and computational precision is 1-bit fixed-point number. When the bitwidth of inputs and weights are both $2$, our MINW-Net is degenerated to TNN, and the inference operation and computational precision are consistent with BNN. In this experiment, \textbf{Eq}~(\ref{sign}) is used for 1-bit bitwidth where the quantized values are constrained to $\{-1,1\}$, \textbf{Eq}~(\ref{tnn}) is used for 2-bit bitwidth where the quantized values are constrained to $\{-1,0,1\}$ and \textbf{Eq}~(\ref{qnn}) is used for 3-bit bitwidth where the quantized values are constrained to $\{-1,-2^{-1},-2^{-2},0,2^{-2},2^{-1},1\}$.

We use several CNN models on CIFAR datasets to evaluate the performance of MINW-Net. Model A is the ConvNet that costs about 1.23 M parameter capacity for 1-bit weights, and it consists of six convolutional layers and three full-connected layers. Model B is derived from Model A by reducing the number of channels by half for all six convolutional layers, which costs about 0.54 M parameter capacity. Model C continues to reduce the number of channels by half for all three fully-connected layers based on Model B, which costs 0.30 M parameter capacity. Model D has the least parameter capacity here that costs 0.21 M by reducing the number of channels by half for all three fully-connected layers based on Model C. These models are trained with a batch size of 256 and a learning rate of 0.01 whose learning rule is ADAM \cite{kingma2014adam} with the exponential decay. The classification accuracy on the listed test set is the results of the model training over 200 epochs.

\renewcommand\arraystretch{1.8}
\begin{table*}[htbp]
	\caption{Comparison of classification accuracy on ImageNet test set with different bitwidths of weights and activations. Top-1 and top-5 accuracy for single-crop evaluation results are given based on AlexNet. Note that the ``Ours" results are implemented by our MINW-Net. Other results are reported by \cite{sze2017efficient}. We quantize the same layers of AlexNet to low-precision, just like BNN~\cite{courbariaux2016binarized}, BC~\cite{courbariaux2015binaryconnect}, TWN~\cite{li2016ternary}, TNN~\cite{mellempudi2017ternary} and DoReFa-Net~\cite{zhou2016dorefa}. Top-1 accuracy for full-precision AlexNet is 56.6\% and top-5 accuracy is 80.2\%.}
	\label{alexNet} 
	\begin{center}
		\begin{tabular}{|c|c|c|c|c|c|c|}
			\hline
			\multicolumn{2}{|c|}{\multirow{2}{*}{\textbf{Reduce Precision Method}}}&
			\multicolumn{2}{c|}{\textbf{Bitwidth}}&\multirow{2}{*}{\textbf{Inference Operation}}&\multirow{2}{*}{\textbf{\makecell{Top-1 Accuracy loss\\ vs. 32-bit float (\%)}}}&\multirow{2}{*}{\textbf{\makecell{Top-5 Accuracy loss\\ vs. 32-bit float (\%)}}}\cr\cline{3-4}
			\multicolumn{2}{|c|}{}&\textbf{Weights}&\textbf{Inputs}& & &\cr
			\hline
			\hline
			\multirow{4}{*}{\textbf{Reduce Weights}}&\textbf{BinaryConnect}~\cite{courbariaux2015binaryconnect}&1&32 (float)&XNOR ADDER&19.8&18.2\cr\cline{2-7}
			&Ours&1&16 (float)&XNOR ADDER&8.0&6.8\cr\cline{2-7}
			&\textbf{Ternary Weight Network}~\cite{li2016ternary}&2&32 (float)&XNOR ADDER&3.7&3.6\cr\cline{2-7}
			&Ours&2&16 (float)&XNOR ADDER&3.0&2.5\cr\hline
			\hline
			\multirow{9}{*}{\textbf{\makecell{Reduce Weights\\ and Inputs}}}&\textbf{Binary Neural Network}~\cite{courbariaux2016binarized}&1&1&XNOR&28.7&29.8\cr\cline{2-7}
			&\textbf{DoReFa-Net}~\cite{zhou2016dorefa}&1&1&XNOR&17.1&19.1\cr\cline{2-7}
			&Ours&1&1&XNOR&21.8&20.1\cr\cline{2-7}
			&\textbf{DoReFa-Net}~\cite{zhou2016dorefa}&1&2&XNOR&10.5&11.4\cr\cline{2-7}
			&Ours&1&2&XNOR&9.8&10.1\cr\cline{2-7}
			&\textbf{Ternary Neural Network}~\cite{mellempudi2017ternary}&2&8 (float)&XNOR ADDER&7.6&7.4\cr\cline{2-7}
			&Ours&2&8 (float)&XNOR ADDER&3.2&2.9\cr\cline{2-7}
			&\textbf{DoReFa-Net}~\cite{zhou2016dorefa}&8 &8 &MAC&3.6&3.4\cr\cline{2-7}
			&Ours&3&8 (float)&SHIFT ADDER&{\bf 0.6}&{\bf 0.7}\cr
			\hline
		\end{tabular}
	\end{center}
\end{table*}

\renewcommand\arraystretch{1.8}
\begin{table}[htbp]
	\caption{Top-1 and top-5 error (\%) with ResNet18 on ImageNet. Note that the ``Ours" results are implemented by our MINW-Net. Other results are reported by \cite{louizos2018relaxed,jung2019learning,zhang2018lq}.}
	\label{ResNet} 
	\begin{center}
		\begin{tabular}{ccccccc}
			\toprule
			\multicolumn{7}{c}{\textbf{ResNet18}} \\
			\hline
			\textbf{Method} & \textbf{\makecell{\# Bits \\ weights/inputs}} & \footnotemark[1] & \footnotemark[2] & \footnotemark[3] & \textbf{Top-1} & \textbf{Top-5} \\
			\hline
			\textbf{Original} & 32/32 & - & - & - & 30.46 & 10.81 \\
			\hline
			\multirow{4}{*}{\textbf{Rounding}} & 8/8 & - & - & - & 30.22 & 10.60 \\
			& 6/6 & - & - & - & 31.61 & 11.32 \\
			& 5/5 & - & - & - & 36.97 & 14.95 \\
			& 4/4 & - & - & - & 78.79 & 57.10 \\
			\hline
			\multirow{3}{*}{\textbf{LQ-Nets}~\cite{zhang2018lq}} & 2/32 & \checkmark & \checkmark & - & 32.00 & 11.80 \\
			& 3/32 & \checkmark & \checkmark & - & 30.70 & 11.10 \\
			& 1/2 & \checkmark & \checkmark & - & 37.40 & 15.50 \\
			\hline
			\multirow{3}{*}{\textbf{SR+DR}~\cite{gupta2015deep,gysel2018ristretto}} & 8/8 & - & - & - & 31.83 & 11.48 \\
			& 6/6 & - & - & - & 40.75 & 16.90 \\
			& 5/5 & - & - & - & 45.48 & 20.16 \\
			\hline
			\multirow{2}{*}{\textbf{LR Net}~\cite{shayer2017learning}} & 1/32 & \checkmark & \checkmark & - & 40.10 & 17.70 \\
			& 2/32 & \checkmark & - & - & 36.50 & 15.20 \\
			\hline
			\multirow{2}{*}{\textbf{ELQ}~\cite{zhou2018explicit}} & 1/32 & - & - & - & 35.28 & 13.96 \\
			& 2/32 & - & - & - & 32.48 & 11.95 \\
			\hline
			\multirow{2}{*}{\textbf{SYQ}~\cite{faraone2018syq}} & 1/8 & \checkmark & \checkmark & - & 37.10 & 15.40 \\
			& 2/8 & \checkmark & \checkmark & - & 32.30 & 12.20 \\
			\hline
			\textbf{TWN}~\cite{li2016ternary} & 2/32 & - & - & - & 38.20 & 15.80 \\
			\textbf{INQ}~\cite{zhou2017incremental} & 5/32 & - & - & - & 31.02 & 10.90 \\
			\textbf{BWN}~\cite{rastegari2016xnor} & 1/32 & - & - & - & 39.20 & 17.00 \\
			\textbf{XNOR-Net}~\cite{rastegari2016xnor} & 1/1 & - & - & - & 48.80 & 26.80 \\
			\textbf{HWGQ}~\cite{cai2017deep} & 1/2 & \checkmark & \checkmark & \checkmark & 40.40 & 17.80 \\
			\textbf{DoReFa-Net}~\cite{zhou2016dorefa} & 1/4 & \checkmark & \checkmark & - & 40.80 & 18.50 \\
			\hline
			\multirow{4}{*}{\textbf{MINW-Net} (Ours)} & 2/32 & - & - & - & 36.36 & 15.10 \\
			& 1/32 & - & - & - & 37.60 & 15.81 \\
			& 1/2 & \checkmark & \checkmark & - & 42.16 & 18.90 \\
			& 1/1 & \checkmark & - & - & 47.87 & 26.04 \\
			\bottomrule
		\end{tabular}
	\end{center}
\end{table} 

The first two lines listed in \textbf{Table}~\ref{cifar100} are the same network structure with 1-bit weights and 1-bit activations, where the first line is the prediction accuracy trained from scratch with our AQE and the second line is the prediction accuracy trained from scratch with the STE. We evaluate the performance of AQE by 1-bit weights and activations MINW-Net on CIFAR100 dataset. As the size of the model decreases, the advantages of using AQE gradually emerge. Due to the reduction in network redundancy, the QNNs using AQE have a higher accuracy than using STE. We achieve the prediction accuracy improvement of 0.0\%, 0.2\%, 1.1\% and 1.5\% on Model A, B, C and D respectively.

The trends in \textbf{Table}~\ref{cifar10}~\ref{cifar100} show that the number of channels affects the prediction accuracy. Although MINW-Net with low-precision weights and activations can cause degradation in prediction accuracy, the tiny degradation in accuracy can be ignored compared to the much reduced resource requirement and increased computational efficiency. For CIFAR10, the best performance is MINW-Net with 3-bit weights and 3-bit activations. We achieve the prediction accuracy degradation of 1.2\%, 1.2\%, 2.1\% and 1.7\% on Model A, B, C and D respectively compared with 32-bit counterparts. For CIFAR100, the best performance is MINW-Net with 3-bit weights and 2-bit activations. We achieve the prediction accuracy degradation of 1.5\%, 3.1\%, 4.0\% and 4.1\% on Model A, B, C and D respectively compared with 32-bit counterparts. When running inference, the inference operation is determined as XNOR, SHIFT or MAC according to the smaller bitwidth between weights and activations, and the inference computational precision is based on the larger bitwidth between weights and activations, where MAC represents multiply accumulate operation.

For ResNet, we use a weight decay of 0.0001, a momentum of 0.9 and BN without dropout. The model is trained with a mini-batch size of 128 and a learning rate of 0.1, divided by 10 at 32k and 38k iterations, and terminates at 64k iterations. Seen from the experiments of ResNet in \textbf{Table}~\ref{quantization_error}, our MINW-Net achieves the test error rates of 7.05\% on CIFAR10, 37.16\% and on CIFAR100, just rises 0.44\% on CIFAR10 and 1.29\% on CIFAR100 compared with 32-bit float ResNet.

For DenseNet, we use a weight decay of 1e-4, a momentum of 0.9 and BN without dropout. The initial learning rate for this model is 0.1, and the model is divided by 10 at 50\% and 75\% of the total number of training periods. And we use a batch size of 64 for a total of 300 periods on CIFAR. Compared with 32-bit float DenseNet, the test error of our MINW-Net on CIFAR10 increases by 0.70\% (from 4.51\% to 5.21\%) and on CIFAR100 increases by 1.43\% (from 22.27\% to 23.70\%), as shown in \textbf{Table}~\ref{quantization_error}.

\subsection{ImageNet}

ImageNet data set is a large image data set organized by professor Fei Fei Li of Stanford university covering all aspects of computer vision. The dataset used for image classification is the ILSVRC2012~\cite{deng2009imagenet} image classification dataset, which identifies the main objects in the image. To further evaluate the effect of our MINW-Net on ILSVRC2012 image classification dataset which contains 1.2 million high-resolution natural images from 1000 categories, we resize these images to 224$\times$224 pixels before input them to the network. In the following experiments, we use top-1 and top-5 accuracy to measure our single-crop evaluation results.

\textbf{AlexNet:} It is the first CNN structure to show success and wide attention on the ImageNet classification task. The architecture consists of 61 million parameters and 65,000 neurons, and is composed of 5 convolutional layers and 2 fully connected layers~\cite{Krizhevsky2012ImageNet}. The output layer is a 1000 channel softmax. We use AlexNet in combination with BN.

In the training, the input image is cropped with a random size, the clipping is adjusted to 256$\times$256 pixels, and then 224$\times$224 images are extracted randomly for training. We train MINW-Net for 50 epochs based on AlexNet, with a batch size of 128. In our AlexNet implementation, we use an ADAM optimizer with faster convergence and the learning rate of 1e-4. We replace the ``Local Contrast Renormalization" layer with the ``Batch Normalization" layer. At inference, we use 224$\times$224 center crop for forward propagation.

The ablation experiments are listed in \textbf{Table}~\ref{alexNet}. It is reported in \cite{rastegari2016xnor} that the accuracy score of baseline AlexNet model is 56.6\% for top-1 and 80.2\% for top-5. In the ablation study, in addition to the difference of quantization methods, we strictly control the consistency of variables such as network structure, bit width and quantization layers. In experiments of ``1-1" v.s. ``1-1" for BNN, ``1-16" v.s. ``1-32" for BC, ``2-16" v.s. ``2-32" for TWN, ``2-8" v.s. ``2-8" for TNN and ``3-8" v.s. ``8-8" for DoReFa-Net, our MINW-Net top-1 accuracy improved by 6.9\%, 11.8\%, 0.7\%, 4.4\% and 3.0\% respectively. For ``3-8" v.s. ``32-32", our MINW-Net only reduces the top-1 accuracy by 0.6\%.

In \textbf{Table}~\ref{ResNet}, we list our MINW-Net results for top-1 and top-5 and compare against multiple recent works on ResNet-18. The ablation experiments employ all fixed-point weights quantization and part fixed-point activations quantization: SR+DR~\cite{gupta2015deep,gysel2018ristretto}, TWN~\cite{li2016ternary}, BWN~\cite{rastegari2016xnor}, XNOR-Net~\cite{rastegari2016xnor}, DoReFa-Net~\cite{zhou2016dorefa}, INQ~\cite{zhou2017incremental}, HWGQ~\cite{cai2017deep}, LR-Net~\cite{shayer2017learning}, SYQ~\cite{faraone2018syq}, ELQ~\cite{zhou2018explicit} and rounding. The instructions of quantization are mentioned explicitly by footnotes.

\section{Conclusion and future work}

In this paper, we have introduced AQE, a novel estimator to propagate asymptotic-estimated gradient through stochastic neuron output, which is correlated with both the original full-precision value and the quantized value, in neural networks involving noise sources. Since our AQE has the asymptotic behaviour, the neuron output will gradually approach the quantized value as the training epochs increase (the actual performance is that the ratio of the original full-precision value is decreased and the ratio of the quantized value is increased). After the training, the output will completely convert to the quantized value. With the hyper-parameter $\alpha=\frac{1}{2}$, we have proven that our AQE will degenerate into STE and unbiased estimator. We have proposed MINW-Net, a quantized neural network with M-bit Inputs and N-bit Weights, which is trained by AQE. All of the activations and weights during forward passes are quantized, because there is no layer is special in MINW-Net. When the smaller bitwidth between weights and activations are 1 or 2, all the convolution operations in inference can be replaced by XNOE operations. Similarly, all the convolution operations are replaced by SHIFT operations when the smaller bitwidth is 3.
\footnotetext[1]{First layer not quantized}
\footnotetext[2]{Last layer not quantized} 
\footnotetext[3]{Modified architecture}

By experiments, we have verified the asymptotic behaviour of our AQE. And the BNN with AQE can achieve a prediction accuracy 1.5\% higher than the BNN with STE. The MINW-Net, which is trained from scratch by AQE, can achieve comparable classification accuracy as 32-bit counterparts on CIFAR test sets. Extensive experimental results on ImageNet dataset show great superiority of the proposed AQE and our MINW-Net achieves comparable results with other state-of-the-art QNNs.

Our future work should explore how to achieve the inference of MINW-Net on FPGAs where the inference operation of XNOR or SHIFT will be applied on hardware platform.


\bibliographystyle{Bibliography/IEEEtranTIE}
\bibliography{Bibliography/IEEEabrv,Bibliography/BIB_1x-TIE-2xxx}\ 

\end{document}